\documentclass{article}

\usepackage[margin=1.in]{geometry}

\usepackage[utf8]{inputenc} 
\usepackage[T1]{fontenc}    
\usepackage{booktabs}       
\usepackage{amsfonts}       
\usepackage{nicefrac}       
\usepackage{microtype}      
\usepackage{xcolor}         
\usepackage{algorithm}
\usepackage{algorithmic}
\usepackage{graphicx}
\usepackage{subcaption}
\usepackage{amsthm}
\usepackage{amsmath}
\usepackage{placeins}
\usepackage{hyperref}
\usepackage{wrapfig}
\usepackage{authblk}

\usepackage{natbib}
\setcitestyle{authoryear,open={(},close={)}}

\title{Conditional Loss and Deep Euler Scheme for Time Series Generation}
\date{}

\author[1,2,3]{Carl REMLINGER}
\author[2,3]{Joseph MIKAEL}
\author[1]{Romuald ELIE}
\affil[1]{Université Gustave Eiffel}
\affil[2]{EDF Lab}
\affil[3]{FiME (Laboratoire de Finance des Marchés de l'Énergie)}

\newtheorem{theorem}{Theorem}[section]

\newtheorem{prop}[theorem]{Proposition}
\newtheorem{remark}[theorem]{Remark}

\begin{document}
\maketitle
\begin{abstract}
We introduce three new generative models for time series that are based on Euler discretization of Stochastic Differential Equations (SDEs) and Wasserstein metrics. Two of these methods rely on the adaptation of generative adversarial networks (GANs) to time series. The third algorithm, called Conditional Euler Generator (CEGEN), minimizes a dedicated distance between the transition probability distributions over all time steps. In the context of Itô processes, we provide theoretical guarantees that minimizing this criterion implies accurate estimations of the drift and volatility parameters. We demonstrate empirically that CEGEN outperforms state-of-the-art and GAN generators on both marginal and temporal dynamics metrics. Besides, it identifies accurate correlation structures in high dimension. When few data points are available, we verify the effectiveness of CEGEN, when combined with transfer learning methods on Monte Carlo simulations. Finally, we illustrate the robustness of our method on various real-world datasets.
\end{abstract}

\section{Introduction}

Time series Monte Carlo simulations are widely used for multiple industrial applications such as investment decisions \citep{kelliher2000using}, stochastic control \citep{pham2009continuous} or weather forecasts \citep{mullen1994monte}. They are notably considered in the financial sector, for market stress tests \citep{sorge2004stress}, risk management and deep hedging \citep{buehler2019deep, fecamp2019risk}, or for measuring risk indicators such as Value at Risks \citep{jorion2000value} among others. Providing Monte Carlo simulations representative of the time series of interest is a difficult and mostly manual task, which requires underlying modeling assumptions about the time dependence of the variables. Hence, it is difficult to update these models when a new type of data is observed, such as negative interest rates, negative electricity prices or unusual weather conditions.  
This naturally calls for the development of reliable model-free data generators for time series.


Generative methods such as Variational Auto Encoders (VAE) \citep{kingma2013auto} or Generative Adversarial Networks (GAN) \citep{goodfellow2014generative} provide state-of-the-art accuracy for the generation of realistic images \citep{xu2018dp} or text \citep{zhang2017adversarial}. The development of similar generative methods for time series is very promising \citep{lyu2018improving,chen2018model}.
However, due to the complex and possibly non-stationary underlying temporal structure of the initial time series, these generative methods, especially GANs, are very difficult to apply as such \citep{yoon2019time}.
Efficient generation of time series requires a proper learning of time-marginals as well as a faithfully representation of the underlying time structure.

In this paper, we embed time series into a discretized Euler approximations of Itô processes, which are characterized by their deterministic drift and volatility parameters. The three proposed generators rely on deep learning approximation of both drift and volatility functions. This representation benefits from a theoretically grounded temporal dynamic and provides a meaningful structure that avoids complex neural network architectures. Moreover, the considered Euler models allow tractable, at least controllable, generator outputs, which can be difficult with deep embedding such as \citep{yoon2019time}. This feature is a key component in industrial applications, especially for decision-making-process.\\
By combining deep Euler representation with Wasserstein distance \citep{villani2008optimal}, we introduce the Euler Wasserstein GAN (EWGAN), inspired by \citep{arjovsky2017wasserstein}. Our second GAN-based-model, called Euler Dual Discriminator (EDGAN) is an adaptation of the DVDGAN presented in \citep{clark2019efficient}. A spatial discriminator focuses on the accuracy of time-marginal distributions, while  a temporal one focuses on the full sequence of generated time-series.
Both methods compute the Wasserstein-1 distance and compete with the state-of-the-art algorithms Time Series GAN (TSGAN)  \citep{yoon2019time} and COTGAN \citep{xu2020cotgan} on both synthetic and real datasets. Nevertheless, all these GAN approaches still have difficulties to capture a proper temporal dynamics of the time series. We remedy to this problem by considering a loss function based on the conditional distributions $\mathcal{L}(X_{.+\Delta t}\,|\,X_{.})$ of the generated time series. We introduce a Conditional Euler Generator (CEGEN) which optimizes a distance between the transition probability distributions at each time step. On the (large) class of Itô processes, we prove that minimizing this metric provides an accurate estimation of both the drift and volatility parameters. 

A numerical study compares the three approaches to state-of-the-art GANs on synthetic and real datasets and shows the performances of our generators. We verify that our generators can learn to replicate Monte Carlo simulations of classical stochastic processes. Synthetic models give access to more reliable metrics (including theoretical), and allow to make connections between model-based Monte Carlo and model-free methods. EWGAN and EDGAN show a similar accuracy than TSGAN or COTGAN and capture more efficiently the time structure dynamics in dimension up to 20. The best performing model, CEGEN, is able to recover the underlying correlation (or independence) structure of time series, even in high dimensions. 
Moreover, we highlight the robustness of CEGEN, when combined with a transfer learning procedure when too few data are available. 
By properly mixing Monte Carlo generated and sparse real data during training, we CEGEN can take advantage of the synthetic simulations to improve its accuracy on generated samples. 

\textbf{Main Contributions:}
\begin{itemize}
    \item A theoretically grounded time series generator CEGEN combining an Euler structure with a dedicated loss on conditional distributions is proposed.
    \item Relying on a similar Euler structure, we also introduce two alternative GAN-like time series generators inspired by \citep{arjovsky2017wasserstein} and \citep{clark2019efficient}. They exhibit close performance to the state-of-the-art TSGAN \citep{yoon2019time} and COTGAN \citep{xu2020cotgan} on marginal metrics, but capture more accurately the dynamic structure of Ito-based time series. 
    \item A thorough numerical study on synthetic and various real world datasets demonstrate the robustness of our generators. Euler models succeed in correctly learning the underlying drifts and volatility structures of synthetic and outperforms the other considered methods on real datasets (accurate correlation structure up to dimension 20...). A transfer learning application when sparse data is available is provided.
\end{itemize}

\section{Related works}
The bootstrap method proposed by \citep{efron1982jackknife} is one of the first purely data-driven attempt to generate time series. Data samples are simply taken randomly with replacement. The scope of this technique is limited as it does not generate additional synthetic data but is based on historical ones. On the opposite, model-free approaches such as GAN allow to learn empirical distribution from data and thus to generate new samples. However, initial GAN proposals focused on the generation of non temporally ordered outputs. GAN's architecture improvement for the time series case is an intensive area of research. For example, WaveGAN \citep{donahue2018adversarial} uses the causal architecture of WaveNet \citep{oord2016wavenet} for unsupervised synthesis of raw-waveform audio. Alternatively, several works consider recurrent neural networks to generate data sequentially and keep memory of the previous time series states \citep{mogren2016c, esteban2017real}.

Time Series GAN (TSGAN) \citep{yoon2019time} introduces a state-of-the-art method for time series generation which stands out by its specific learning process. At each time step, an embedding network projects time series samples onto a latent space on which a GAN operates. TSGAN manages to get the correct marginal distributions and temporal correlation on classical processes and is used as a baseline in this paper. This method lacks of theoretical foundations ensuring a reliable quality of generated samples. As the usage of generating model-free method grows rapidly, their application to sensitive fields (e.g. finance) must be considered cautiously and requires theoretical and empirical guarantees on the behavior of these generators. For this purpose, an active line of research looks towards reliable embedding of time series, such as signature \citep{fermanian2019embedding, buehler2019deep} or Fourier representation \citep{steinerberger2018wasserstein}.\\
Most recent applications on video generation focus on specific GAN architectures to capture the spatial-temporal dynamics. For example, MoCoGAN \citep{tulyakov2018mocogan} and DVD GAN \citep{clark2019efficient} combines two discriminators, one for the temporal dynamic and another one on each static frame. Specialized generator structures have also been designed, TGAN \citep{saito2017temporal} proposed to generate a dynamic latent space and VGAN \citep{vondrick2016generating} combines two generators, one for marginals and another one for temporal dependencies. Following the idea of applying optimal transport to GANs \citep{arjovsky2017wasserstein,genevay2018learning}, COTGAN \citep{xu2020cotgan} uses causal optimal transport for video sequence generation. To do so, the discriminant penalizes not-causal optimal transport plans, ensuring  that the generator minimizes an adapted Wasserstein distance for time series. This approach benefits of solid theoretical foundations but still lacks of reliable empirical success for noisy time series generation.



\medbreak
\section{Problem formulation}\label{sec:problem}

We aspire to design a time series generator which combines accurate estimation of time-marginal distributions while properly capturing temporal dynamics. 
The generator we propose is designed to be simple enough to be tractable (in the sense that outputs could be controlled) and theoretically grounded. 
To do so, we feed our algorithms with training time series data and seek to learn an empirical probability distribution that best approximates the data one. This task can be tricky, depending on the sequences lengths, the dimension, and the shape of the data distribution.

Although the idea of a model-free approach is attractive, we restrict ourselves to the context of Itô processes. This class of processes encompasses a wide range of time series and yet allows us to develop tractable models based on theory. In addition to providing a robust theoretical framework and controls on the processes generation, Itô processes allow to measure the accuracy of our generators on synthetic samples via closed form expressions or Monte Carlo simulators. 
In comparison to the classical literature \citep{wiese2020quant, buehler2020data}, we do not assume the time series $X$ to be stationary and allow ourselves to consider not-stationary sequences.

\paragraph{Itô process}
We are given i.i.d. samples of a time series, considered as a random vector $X=(X_{t_i})_{i=1\ldots N}$ on $\mathbb{R}^{ d \times N}$, starting from a point $X_0\in\mathbb{R}^{d}$ and observed on a time grid $\mathcal{T} := \{0=t_0 < t_1< ... < t_N=T\}$. For the sake of simplicity, in the following,  we assume a regular time grid with mesh size $\Delta t$.
The discrete time samples are supposed to be drawn from a continuous time underlying process $X$ having the following Itô dynamics: 
\begin{eqnarray}\label{eq:Ito}
dX_t &=& b_X(t,X_t) d t + \sigma_X (t,X_t) dW_t\,,
\end{eqnarray}
where $b_X:\mathbb{R}\times \mathbb{R}^d\rightarrow \mathbb{R}^d$ is the drift term, $\sigma_X:\mathbb{R}\times \mathbb{R}^d\rightarrow \mathcal{M}_{d\times d}$ the volatility term and $W$ is a $d$-dimensional Brownian motion. The parameters $b_X$ and $\sigma_X$ are supposed to satisfy the usual Lipschitz conditions \citep{ikeda2014stochastic} ensuring existence and uniqueness of the solution of Eq.\eqref{eq:Ito}.


\paragraph{Deep Euler representation}\label{sec:deep} 
Samples of $X=(X_{t_i})_{i=1\ldots N}$ are drawn from the continuous time Itô process with dynamics \eqref{eq:Ito}, and can approximately be viewed as samples drawn from the Euler discretization scheme of \eqref{eq:Ito} given by 
\begin{eqnarray}
 X_{t_i+\Delta t} =  X_{t_i} + b_X(t_i, X_{t_i}) \Delta t \sigma_X (t_i, X_{t_i}) \Delta W_{t_i},  \label{eq:eulerX}
\end{eqnarray}
where $(\Delta W_{t_i})_i$ is a collection of i.i.d. $\mathcal{N}(0, \Delta t I_d)$ random variables. We rely on this approximation and introduce the following deep Euler representation of the time series. 
Starting at $t_0=0$, from $Y^{\theta}_0 = X_0$ we generate time series by the following scheme:
\begin{eqnarray}
Y^\theta_{t_i+\Delta t} &=& Y^\theta_{t_i} +  b^\theta_Y (t_i,Y^\theta_{t_i}) \Delta t +  \sigma^\theta_Y(t_i,Y^\theta_{t_i}) Z_{t_i}\,, \label{eq:eulerY}
\end{eqnarray}
where $Z_{t_i}$ are $\mathcal{N}(0, \Delta t I_d)$ i.i.d. random variables and the functions $ b_Y^\theta$ and $ \sigma^\theta_Y$ are $\theta$-parametrized functions approximated by neural networks. 
Our objective is to learn $ b_Y^\theta$ and $ \sigma^\theta_Y$, so that the distribution of the processes $Y^\theta$ and $X$ are close. 
 
\paragraph{Evaluation}\label{sec:eval}
During the learning phase, neither $b_X$ nor $\sigma_X$ are given as inputs to any of the proposed models. However, the proposed formulation allows to compare \textit{a posteriori} $b_X$ and $\sigma_X$, when they are known, to the estimated $ b_Y^\theta$ and $ \sigma^\theta_Y$. This provides a reliable metric on the  generation accuracy.\\
Moreover, this setup provides a convenient way to control the drift $b^\theta_Y$ and volatility $\sigma^\theta_Y$ functions. This task is delicate with deep embedding proposals.
As already mentioned, it is highly challenging for generators of temporally ordered data to create samples  accurate on time-marginal distributions as well as  temporal dynamic metrics. In order to remedy to this weakness, we introduce below 
an innovative loss function based on a distance between conditional distributions at each time step.

\section{Euler Generators}
Euler Generators proposed in this paper are composed of two main elements: a network generating the drift and volatility terms of an Itô process and a distance between distributions to be minimized. The Itô structure facilitates the time series construction, while the distributional distance focuses on the law accuracy of the generated sequences. Both GAN-based and Deep Conditional methods described hereafter share this design. 

\subsection{Euler Generative Adversarial Networks}\label{sec:gangens}
We propose two adaptations of GANs to time series that are based on the deep Euler representation presented in Eq.\eqref{eq:eulerY}. The Wasserstein GAN \citep{arjovsky2017wasserstein} seems to get rid of stability problems encountered in learning (mainly mode collapse) by adapting to the geometry of the underlying space. Our two GAN-based models are built on this proposal and both minimize (differentiable) Wasserstein-1 ($\mathcal{W}_1$) distance. The Rubinstein-Kantorovich duality allows to rewrite the $\mathcal{W}_1$ distance between two random variables $Z_1$ and $Z_2$ in the following way\footnote{$||f||_L$ denotes the smallest Lipschitz constant of $f$.}:
\begin{eqnarray}\label{eq:wasserstein1}
\mathcal{W}_1(\mathcal{L}(Z_1),\mathcal{L}(Z_2))  = \sup_{||f||_L\leq 1} \mathbb{E}_{Z_1\sim \mathcal{L}(Z_1)} \left[f(Z_1)\right] -  \mathbb{E}_{Z_2 \sim \mathcal{L}(Z_2)}\left[f(Z_2)\right].
\end{eqnarray}

\paragraph{Euler Wasserstein GAN (EWGAN)} This model considers a Wasserstein GAN, where the generator relies on the Deep Euler representation \eqref{eq:eulerY} and optimizes the corresponding parameter $\theta$. The discriminator $d_\varphi$ parametrized by $\varphi$ tries to find the optimal $1$-Lipschitz function allowing to compute $\mathcal{W}_1(\mathcal{L}(X),\mathcal{L}(Y^\theta))$ using  the Rubinstein-Kantorovich duality in \eqref{eq:wasserstein1}. The $1$-Lipschitz property of $d_\varphi$ is guaranteed using the gradient penalty trick mentioned in \citep{gulrajani2017improved}. The pseudocode of EWGAN is given in Alg.\ref{alg:ewgan} and details are provided in Appendix \ref{sec:hyperparameters}. Overall, EWGAN minimizes the $\mathcal{W}_1$ distance between the distributions of the original $X=(X_t)_{t_in\mathcal{T}}$ and the generated one $Y^\theta=(Y^\theta_t)_{t_in\mathcal{T}}$:
\begin{eqnarray}\label{eq:EWGAN}
\inf_\theta \mathcal{W}_1(\mathcal{L}(X),\mathcal{L}(Y^\theta)) =
\inf_\theta \sup_{\varphi} \mathbb{E}_{X\sim \mathcal{L}(X)} \left[d_{\varphi}(X)\right]
-  \mathbb{E}_{Y^\theta \sim \mathcal{L}(Y^\theta)}\left[d_{\varphi}(Y^\theta)\right].
\end{eqnarray}

\paragraph{Euler Dual Discriminator (EDGAN)} Our second GAN-based-model, called EDGAN, is an adaptation of the Dual Video Discriminator (DVD) GAN \citep{clark2019adversarial}. DVD GAN uses attention networks and two discriminators in order to generate high fidelity videos. While the spatial discriminator focuses on time marginals and critics images in high resolution, the temporal one considers the full sequence of images in low resolution. We adapt these ideas to our context by considering in EDGAN, a similar dual discriminator architecture while the generator creates samples using the Deep Euler representation in \eqref{eq:eulerY}. Our temporal discriminator follows a similar behavior as the one of EWGAN and focuses on $\mathcal{W}_1(\mathcal{L}(X),\mathcal{L}(Y^\theta))$. In the same time, our marginal discriminator focuses the computation of the $W_1$ distance between marginal distributions $\mathcal{W}_1(\mathcal{L}(X_t),\mathcal{L}(Y^\theta_t))$, for each $t\in\mathcal{T}$. The details and pseudocode of EWGAN are given in Alg.\ref{alg:edgan} and provided in Appendix \ref{sec:hyperparameters}.

\subsection{Conditional Loss Method}\label{sec:deepconditionalmethod} 


\subsubsection{A loss function based on conditional distributions}\label{sec:road}
The difficulty arising when trying to design a loss function for a time series generator comes from the need to get the correct balance between the marginal distribution fitness and the good representation of the temporal structure. On the one hand, we cannot only focus on marginals because having $\mathcal{L}(X_{t_i}) \sim  \mathcal{L}(Y^\theta_{t_i})$ for all $t_i\in \mathcal{T}$ does not imply that $b_X = b_{Y^\theta}$ nor that $\sigma_X =\sigma_{Y^\theta}$ (see the counterexample in Appendix \ref{sec:counterexample}). On the other hand, instead of working on marginals, one can wonder if considering time series realization as samples of a vector defined on $\mathbb{R}^{n+1}$ provides better results. Unfortunately, and as mentioned in \citep{yoon2019time}, learning the joint distribution $\mathcal{L}(X_{t_0}, \ldots, X_{t_n})$  may not be sufficient to guarantee that the network captures the temporal dynamics, even with memory-based networks. An empirical example of unsatisfactory generation based on joint law is illustrated in Figure in Appendix \ref{fig:sinkhorn}, the generated trajectories are smooth.
In the case of time series, one should simply refrain from applying a loss based only on marginal or joint distributions. To provide a reliable solution to this issue, we propose to focus on the transition probabilities at each time step by conditioning on the previous state. Moreover, by doing so, we are able to produce theoretical results on Itô coefficient estimation accuracy.

\subsubsection{CEGEN Algorithm}\label{sec:cegen}
Contrarily to the previous GAN-based generators, CEGEN does not require a discriminator network. The idea consists in considering a loss function that compares the conditional distributions $\mathcal{L}(Y^\theta_{t_{i+1}}\,|\,Y^\theta_{t_{i}})$ with $\mathcal{L}(X_{t_{i+1}}\,|\,X_{t_{i}})$, for each time step $t_{i}\in\mathcal{T}$. The latter conditional distributions are Gaussian when considering Euler-discretized Itô processes. We consider the following metric: 
\begin{eqnarray}
\mathcal{W}_2^2(\mathcal{L}(X), \mathcal{L}(Y))=
\| \mathbb{E}[X] - \mathbb{E}[Y] \|_2^2
+ \mathcal{B}^2(Var(X), Var(Y)) \label{eq:newwasserstein}
\end{eqnarray}
where $\mathcal{B}$ is the Bures metrics \citep{bhatia2019bures,malago2018wasserstein} defined by
$\mathcal{B}^2(A, B) {=} Tr(A) + Tr(B) - 2 Tr(A^\frac{1}{2}BA^\frac{1}{2})^{1/2}$, for positive definite matrices $A$ and $B$. If $X$ and $Y$ are gaussian, $W_2$ is the definition of the Wasserstein-2 distance \citep{gelbrich1990formula}. This metric \eqref{eq:newwasserstein} captures meaningful geometric features between distributions, and $\mathcal{W}_2$ transportation plan is very sensitive to the outliers thus increases the distribution estimation accuracy. The Bures formulation \eqref{eq:newwasserstein} allows us to compute exactly the Wasserstein-2 distance, instead of regularized ones \citep{genevay2018learning, cuturi2013sinkhorn}.\\
Moreover, we want to have a theoretically grounded methodology and the Bures metric allows us to provide guarantees that minimizing the conditional loss implies accurate estimation for the drift and diffusion coefficients. 
We see that whenever the conditional distributions of the form $\mathcal{L}(X_{t_{i+1}}\,|\,X_{t_{i}}=z)$ and $\mathcal{L}(Y_{t_{i+1}}\,|\,Y_{t_{i}}=z)$ coincide in ${\mathcal{W}}_2$, the drift and diffusion parameters coincide as well (see Prop. \ref{prop:conditioning} in Appendix).
This is encouraging but in general conditioning from the very same point is complicated. Proposition \ref{prop:inegalite} extends this property when the previous states belong to a small ball around $z$. 

To build up our generator, we create at each time $t_i$ a partition $(I_{k})_{k\le N_k}$ of the union of supports of $X_{t_{i}}$ and $Y^{\theta}_{t_{i}}$. 
For a given batch of samples, $\mathcal{L}(X_{t_{i+1}}\,|\,X_{t_{i}}\in I_k)$  is approximated by extracting the elements $X_{t_{i+1}}$  such that $X_{t_{i}}\in I_k$. $\mathcal{L}(Y_{t_{i+1}}\,|\,Y_{t_{i}}\in I_k)$ is approximated in the same way. The ${\mathcal{W}}_2^2$ metric between the two conditional distributions are then summed up over all $K$ subdivisions and over all time steps:
\begin{eqnarray}\label{Conditional_Loss}
l(X, Y^\theta) = \sum_{i=0}^{N-1} \sum_{k=1}^{N_k} &{\mathcal{W}}^2_2(\mathcal{L}(X_{t_{i+1}}| X_{t_i} \in I_k), \mathcal{L}(Y^\theta_{t_{i+1}}| Y^\theta_{t_i} \in I_k))\nonumber
\end{eqnarray}
The pseudocode of CEGEN is given in Alg.\ref{alg:cgen} and details are provided in Appendix \ref{sec:hyperparameters}.
Observe that our framework boils down to computing ${\mathcal{W}}_2$ metric between empirical distributions.
Bures metrics is computed using the Newton-Schulz method \citep{muzellec2018generalizing}, which is a differentiable way to get covariance matrice square roots. 

\begin{algorithm}[h]
  \caption{Algorithm CEGEN.}
  \label{alg:cgen}
  \footnotesize
  \begin{algorithmic}

    \STATE {\bfseries Input:} $\mathcal{D}$ samples of $X$, $m$ batch size, $K$ Nb of subdivisions, $\gamma$ learning rate\\
    \STATE {\bfseries Initialize:} $\theta$ (randomly picked)
    \WHILE{Not converged}
      \FOR{$t_i=0...T$}
        \STATE Sample $m$ observations $(x_{t_i +1})$ from of ${X_{t_i+1}}$
        \STATE Sample $z \sim \mathcal{N}(0,I_D \Delta t)$
        \STATE $y_{t_i +1} \leftarrow y_{t_i} + g^b_\theta(t_i,y_{t_i}) \Delta t + g^\Sigma_\theta(t_i,y_{t_i}) z$
            \STATE $I_K\leftarrow  K$ subdivisions of Supp$(X_{t_i})\cup$  Supp$(Y_{t_i})$
            \FOR{$k=0...K$}
                \STATE $\ell_{t_i+1, k}\leftarrow {\mathcal{W}}^2_2(\mathcal{L}(x_{t_i+1}| x_{t_i} \in I_k), \mathcal{L}(y_{t_i+1}| y_{t_{i}} \in I_k))$
            \ENDFOR
        \ENDFOR
      \STATE $\theta = \theta  - \gamma \nabla_{\theta}  \sum_{t_i=1}^{T-1} \sum_{k=1}^{K} \ell_{k,t_i+1}$ 
    \ENDWHILE
    \STATE {\bfseries Output: $y$}
\end{algorithmic}
\end{algorithm}

\subsubsection{Theoretical guarantee}\label{sec:prop}
In order to theoretically ground the choice of a loss function between conditional distributions based on ${\mathcal{W}}_2$, we need to quantify how reducing the ${\mathcal{W}}_2$ loss (Eq. \ref{eq:newwasserstein}) implies proximity between drift and volatility parameters. An analysis on the topic is provided in Appendix \ref{Appx A}. The following result is allowed by the specific expression of the loss (\ref{eq:newwasserstein}) implemented in CEGEN.

\begin{prop} \label{prop:inegalite} 
Assume that $\sigma^2_X(t_i, .)$, $\sigma^2_{Y^\theta}(t_i, .)$ are strictly positive and, together with $b_X(t_i, .)$ and $ b_{Y^\theta}(t_i, .)$, $K$-Lipschitz in their second coordinate. 
For $t_i\in\mathcal{T}$, let $(I_k)_k$ be a regular partition covering Supp$(X_{t_i})\cup$ Supp$(Y_{t_i})$ with mesh size $\Delta x$.
Let $\varepsilon>0$.\\
If
${\mathcal{W}}^2_2\left(\mathcal{L}(X_{t_{i+1}}|X_{t_i}\in I_k), \mathcal{L}(Y^\theta_{t_{i+1}}| Y^\theta_{t_i}\in I_k)\right)\leq \varepsilon^2$ for any $k$, 
then, for $z$ in the partition
\begin{eqnarray*}
\|b_X(t_i, z) - b_{Y^\theta}(t_i, z) \|_2  &\leq& \frac{\varepsilon + \Delta x}{\Delta t} + 2K\Delta x.
\end{eqnarray*} 
Furthermore if $d=1$, 
\begin{eqnarray*}
\| \sigma_X(t_i,z)- \sigma_{Y^\theta}(t_i, z)\|_2 &\leq&  \varepsilon/\sqrt{\Delta t}\nonumber + 2K\Delta x.
\end{eqnarray*}
and, when $d>1$ and $Tr(\sigma^2_X(t_i, z)) = Tr(\sigma^2_{Y^\theta}(t_i, z)) = \alpha$, we have
\begin{eqnarray*}
\|\sigma_X(t_i, z) -  \sigma_{Y^\theta}(t_i, z)\|_2 &\leq& \sqrt{\frac{2\alpha}{\Delta t }}\varepsilon + 2K\Delta x.
\end{eqnarray*}
\end{prop}
As described in \ref{sec:proofconditionalwasserstein}, the previous result is proved using useful inequalities between Hellinger and Bures distances. The $\alpha$ coefficient comes from the need of using density matrices, in practice one can easily normalize covariance matrices by their traces.
Proposition \ref{prop:inegalite} implies that by conditioning over sufficiently small intervals, a low ${\mathcal{W}_2}$ loss between transition distributions guarantees a good diffusion and drift representation. 



\section{Numerical Study}\label{sec:numerical}
We now turn to the numerical evaluation of EWGAN, EDGAN and CEGEN in comparison to the state-of-the-art TSGAN and COTGAN, on various synthetic and real time series. Neural network architectures and hyper-parameters are described in Appendix \ref{sec:hyperparameters}. 


\subsection{Datasets} 
Two kinds of datasets are used: synthetic and real time series dataset.
In single dimension, we use Black\&Scholes (BS) model ($d X_t =r X_t dt + \sigma X_t dW_t$) and an Ornstein-Uhlenbeck (OU) model ($d X_t = \theta(\mu-X_t) dt + \sigma dW_t$). For these two stochastic models, our empirical references are drawn from Monte Carlo (MC) simulators. The simulations are performed on a regular time grid of $30$ dates, the maturity is $0.25$ (1 simulation per day for 3 months) and $X_0 = 0.2$. BS model (resp. OU) has coefficients of $r=0.8$, $\sigma=0.3$ (resp. $\sigma=0.1, \mu=0.6$ and $\theta=7$).
In higher dimensions, we proceed with the same methodology but with multivariate correlated BS time series ($d$ = 4, 10, 20).
The real datasets include various nature of time series and are detailed in Appendix \ref{sec:data}.

\subsection{Evaluation metrics}
\label{sec:evaluation}
We consider several metrics to evaluate the accuracy of the generators. For all metrics the lower, the better.

\paragraph{(1) Marginal metrics.} These metrics quantify the quality at each time step of the marginal distributions induced by the generated samples in comparison to the empirical ones. This includes Fréchet Inception Distance (FID) \citep{heusel2017gans} as well as classical statistics (mean, 95\% and 5\% percentiles, maximum and minimum denoted respectively Avg, q95, q05, Max, Min). We systematically compute the mean squared error (MSE) over time of these statistics between the real and generated samples. This helps measuring whether a generator manages to get an accurate overall envelope of the processes.

\paragraph{(2) Temporal dynamics.} This metric aims at quantifying how the generator is able to capture the underlying time structure of the signal. For this purpose, we compute the difference between the quadratic variations of both reference and generated time series. The quadratic variation (QVar) of an Itô process $X$ is given by $[X]_t=\int_0^t\sigma^2_X(s,X_s) ds$. Thus the temporal metric ensures that the diffusion $\sigma_X$ is well estimated too. We compute $[X]_t$ in the discrete case with $\sum_{i} |X_{t_{i+1}}-X_{t_{i}}|^2$.

\paragraph{(3) Correlation structure.} The metric denoted \textit{Corr} in the following is the term-by-term MSE between empirical correlation from reference samples on one side and from generated samples on the other side. It evaluates the ability of a generator to capture the multi-dimensional structure of the signal. 

\paragraph{(4) Underlying process parameters.} A by-product output of Euler-based generators are the estimated drift $ b^\theta_Y(.)$ and diffusion $ \sigma^\theta_Y(.)$ coefficients of the generator. When using synthetic data, we can compare the true underlying processes parameters to the estimated ones. In the BS case, the drift and volatility coefficients are estimated by the empirical average of $( b
^\theta_Y(t, Y^\theta_t)/Y^\theta_t)_{t\in\mathcal{T}}$ and $( \sigma
^\theta_Y(t, Y^\theta_t)/Y^\theta_t)_{t\in\mathcal{T}}$. In the OU case, $\sigma^\theta_Y$ is estimated in a similar manner, while $\theta$ and $\mu$ are estimated by regressing $ b^\theta_Y$ on $(t,Y^\theta_t)$. These statistics cannot be computed in the same way with TSGAN due to its specific deep embedding, nor COTGAN.

\textbf{(5) Discriminative and predictive scores.} We use two distinct scores, as proposed in \citep{yoon2019time}. First, we train a classification model (a 2-layer LSTM) to distinguish real sequences from the generated ones. The accuracy of the classifier provides the discriminative score. Second, the predictive score is obtained by training a sequence-prediction model (a 2-layer LSTM) on generated time series to predict the next time step value over each input sequence. Performance is measured in terms of MAE.

\FloatBarrier
\subsection{One-dimensional simulated process (Exp. A)}\label{sec:onedim}

\begin{table}
\centering
\begin{tabular}{cccc}
\hline 
\footnotesize
& CEGEN&  EWGAN&  EDGAN \\ \hline\hline
\multicolumn{4}{c}{\textbf{Black-Scholes}}\\ 
$\hat{r}$ (0.8)&     \textbf{0.739}&  0.581&         0.996\\
$\hat{\sigma}$  (0.3)&     0.324&  \textbf{0.314}&         0.379\\ \hline
  \multicolumn{4}{c}{\textbf{Ornstein-Uhlenbeck}}\\ 
${\theta}$ (7.0)&    \textbf{7.05}&         4.36&            4.68\\
$\hat{\mu}$ (0.6)   &    \textbf{0.60}&         0.75&            0.72\\
$\hat{\sigma}$ (0.1)&    \textbf{0.11}&         0.16&           0.02\\ \hline 
\end{tabular}
\caption{\textbf{Exp. A. }Model parameter estimations for drift and volatility function.}
\label{tab:estimateddrfitandsigma}
\end{table}
We start with a unidimensional time series by comparing the five generators in the OU case. Figure \ref{fig:ousamples} illustrates how crucial is the balance between the estimation of the marginal distributions and the temporal structure.
On the one hand, the trend and marginal distributions of the time series generated by both GANs seem close to the empirical benchmark. However, the temporal dynamics between two time steps is not respected as confirmed by the QVar metric in Table \ref{tab:envelope} (in Appendix).
On the other hand, CEGEN model manages to capture the overall envelope and is able to fit the dynamics of time series as the QVar metrics highlights in Table \ref{tab:envelope}.
Table \ref{tab:estimateddrfitandsigma} reports the reference drift and volatility coefficients with those obtained by the three Euler-based generators. We can see a good estimation of the CEGEN method and to a lesser extent of the EWGAN method while EDGAN fails to estimate the parameters correctly. Euler structure alone does not manage to recover the right parameter values.
To conclude this section, it appears that regarding the overall dynamics and the marginals, CEGEN is a reliable generator of time series. The question we address in the next section is how CEGEN scales to higher dimensions.

\begin{figure}[H]
    \centering
    \includegraphics[width=.25\linewidth]{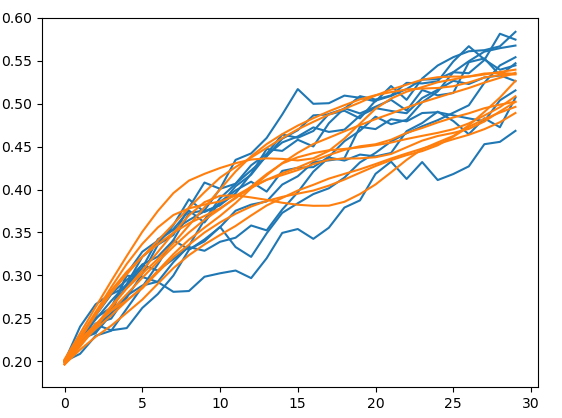}
    \includegraphics[width=.25\linewidth]{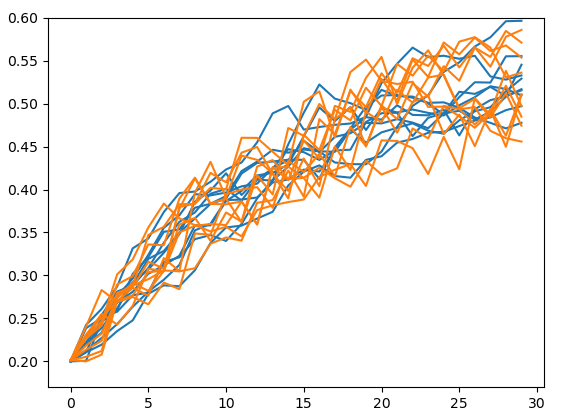}
    \includegraphics[width=.25\linewidth]{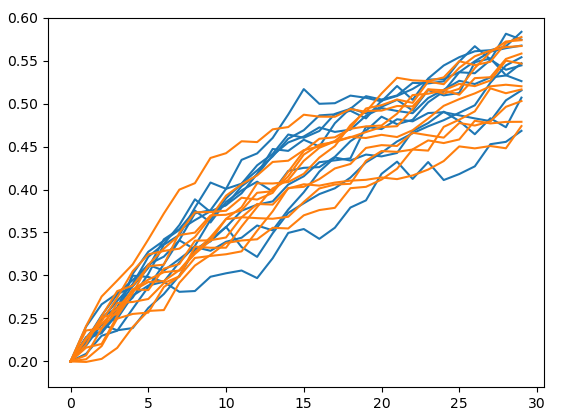}
    \caption{\textbf{Exp. A.} Example of Ornstein-Uhlenbeck samples (in blue) with COTGAN, TSGAN and CEGEN generations (orange).}
    \label{fig:ousamples}
\end{figure}


\FloatBarrier
\subsection{Scaling the dimension (Exp. B)}\label{sec:multidim}
\begin{table}[H]
    \centering
    \resizebox{0.6\textwidth}{!}{
    \begin{tabular}{c ccccc}
    \hline
    Dim& CEGEN&    EWGAN&      EDGAN&      TSGAN&       COTGAN\\ \hline\hline
    4  &   \textbf{.007}&.015&.053&.177&.031\\
    10 &   \textbf{.011}&.055&.022&.259&.035\\
    20 &   \textbf{.006}&.034&.014&.481&.019\\ \hline
    \end{tabular}
    }
    \caption{\textbf{Exp. B} \small MSE between reference and generator empirical correlation matrices on Black-Scholes.}
    \label{tab:correlation}
\end{table}
Table \ref{tab:correlation} reports the discrepancies between reference empirical correlation and generated time series correlation for dimension $d=4,10,20$. We can see that in dimension up to 20, CEGEN obtains a significant improvement compared to all the GANs. 
Figure \ref{fig:scorescaleplot} illustrates how well the CEGEN generator outperforms the other generators with respect to the FID and QVar metrics in the higher dimensions.
This is confirmed by statistics on volatility and drift, as well as
by envelope statistics described in Table \ref{tab:envelopedimensions} in Appendix. Figure \ref{fig:egenbures20samp} shows that the 20 processes envelopes are well respected by CEGEN. 
These good global performances encourage us to focus on the conditional generator in the following transfer learning section.
\begin{figure}[h]
    \centering
    \includegraphics[width=0.4\linewidth]{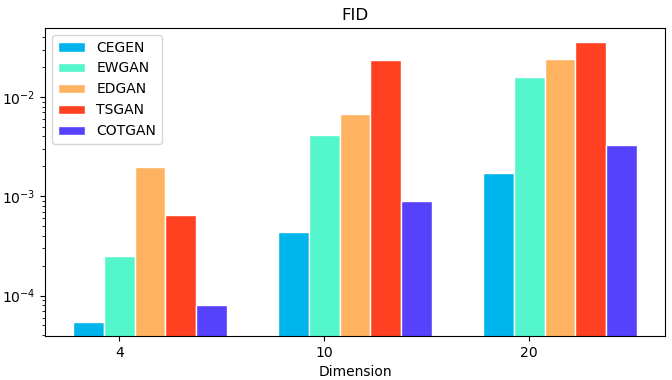}
    \includegraphics[width=0.4\linewidth]{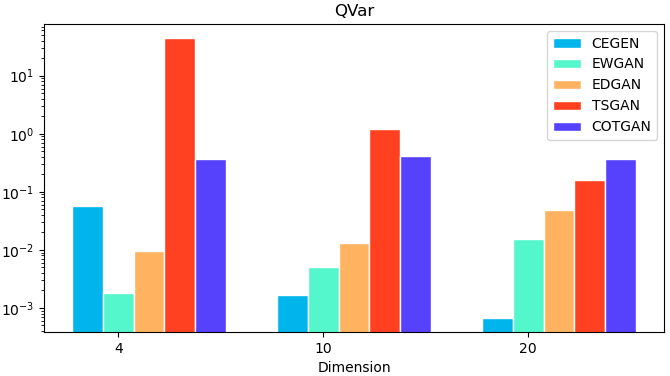}
    \caption{\textbf{Exp. B. } \textbf{Left:}  Average of Fréchet Inception Distance between distributions at each time step. \textbf{Right:} MSE between quadratic variations. (Log scale).}
     \label{fig:scorescaleplot}
\end{figure}

\FloatBarrier
\subsection{Transfer learning for small dataset (Exp. C)}\label{sec:transferfrugal}
Deep generators may need more data than available to be trained effectively. As is done in transfer learning \citep{torrey2010transfer}, we propose to start the training with a reasonably wrong model and to finish up the training with the few real data samples. This situation is tested on synthetic data and allows us to track the drift and volatility parameters evolution during the training phase.
The reference data are assumed to come from samples simulated from an OU process, while the wrong but reasonable Monte Carlo samples come from a misspecified OU model. The original (resp. misspecified) parameters are $\sigma=0.15$, $\mu=0.6$, $\theta=2.0$ (resp. $\sigma_{MC}=0.1$, $\mu_{MC}=0.8$, $\theta_{MC}=3.0$) and original data sequences include only 60 sequences of 30 dates (5 years of monthly measures).
The CEGEN with transfer is compared with a CEGEN only trained with the few available real sequences.

Figure \ref{fig:adpatparam} provides the coefficient evolution of both generators during the training process. The transfer iteration time is represented by the red vertical line. Firstly trained with miscalibrated OU model, the transfer learning approach is able to retrieve the parameters when fed with few samples of the target model. The generator only trained with few real samples is unable to estimate correctly the $\theta$ and $\sigma$ coefficients, but exhibits a better  estimation of $\mu$. The CEGEN benefiting from the transfer learning takes advantage of the initial training phase and provides an overall better estimation. 

\begin{figure}
    \centering
    \includegraphics[width=0.32\linewidth]{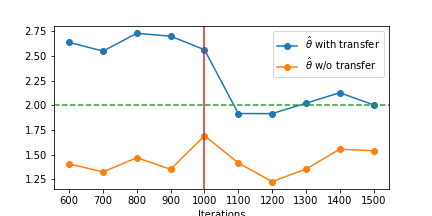}
    \includegraphics[width=0.32\linewidth]{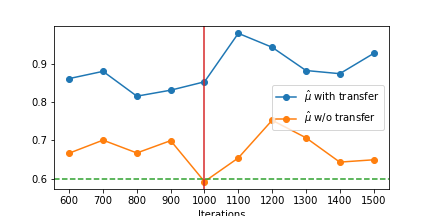}
    \includegraphics[width=0.32\linewidth]{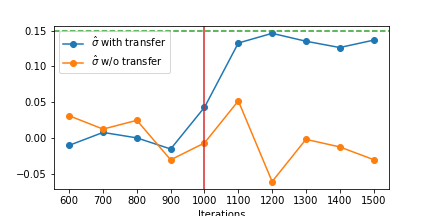}
    \caption{\textbf{Exp. C. } \small Evolution of parameter estimations during training when a transfer occurs at iteration 1000 (red lines). The dashed green lines correspond to the theoretical target value of the real model. The orange lines indicates coefficient estimation of CEGEN only trained on few data and blue lines CEGEN which is first trained on misspecified model then with few real data.} 
    \label{fig:adpatparam}
\end{figure}

This framework is a way to update an existing model with the help of incoming real data. In this situation, the training would start with samples simulated from the consistent model and end with real world inputs.

\subsection{Experiments on real-world datasets (Exp. D)}\label{sec:realworld}
Finally, we test the CEGEN algorithm on various real data with heterogeneous time series.
In Table \ref{tab:realdata}, we compare CEGEN performances with the help of FID, QVar and Corr. Our model outperforms GANs or are close in term of FID and QVar for each real times series, and captures well the correlation structure of the majority of the signals. However, some QVar from TSGAN or COTGAN are lower than CEGEN despite their generated trajectories are significantly smoother than real data. To better evaluate the fidelity of the generation we need to consider \textit{post-hoc} metrics. 
Table \ref{tab:dispredscores} reports discriminative and predictive scores for all models (except EWGAN where scores can be found in Appendix), the lower the better. Our generator almost consistently generates higher-quality time series in comparison to the benchmark. On Electric Load data, COTGAN is better able to capture seasonality of the times series, but generates too smooth trajectories. In the opposite CEGEN proposes more faithful times series in term of noise, but struggles to fool the classifier.

\begin{table}[h!]
    \centering
    \begin{tabular}{l | ccc | ccc}   
    \hline
    Data&\multicolumn{3}{c|}{CEGEN}&\multicolumn{3}{c}{EDGAN}\\
    &                   FID & QVar & Corr & FID & QVar & Corr\\\hline\hline
    Spot prices (d=2) & 1.38e-04&2.09e+00&\textbf{2.10e-02} & 3.11e-03&2.18e+00&4.12e-02\\
    Stocks (d=6) & \textbf{1.04e-04}&2.10e+01&2.33e-03 & 7.93e-03&2.43e+01&9.78e-03\\
    Electric Load (d=12) &6.47e-03&4.30e+00&\textbf{1.27e-03}  &  4.62e-02&1.27e+00&1.56e-03\\
    Jena climate (d=15) &\textbf{1.10e-03}&\textbf{7.18e+00}&\textbf{1.75e-02}  & 4.39e-02&7.73e+00&1.46e-01\\
    \hline
    \hline
    Data&\multicolumn{3}{c|}{TSGAN}&\multicolumn{3}{c}{COTGAN}\\
    &                   FID & QVar & Corr & FID & QVar & Corr\\\hline\hline
    Spot prices (d=2) & 2.12e-04&\textbf{9.00e-02}&4.45e-02  &  \textbf{1.09e-04}&8.25e-01&4.15e-02 \\
    Stocks (d=6) & 3.46e-03&2.19e+01&2.76e-01  &  1.49e-04&\textbf{1.86e+01}&\textbf{1.62e-03}    \\
    Electric Load (d=12) &  \textbf{5.12e-03}&\textbf{9.18e-01}&1.87e-03 & 4.10e-01&3.45e+00&6.27e-01\\
    Jena climate (d=15) &    4.07e-03&8.49e+01&1.89e-02 & 4.48e-03&7.90e+00&2.34e-02 \\
    \hline
    \end{tabular}
    \caption{\textbf{Exp. D. }   Accuracy evaluations for generations on real world time series (the lower, the better).}
    \label{tab:realdata}
\end{table}

\begin{table}[h!]
    \centering
    \footnotesize
    \begin{tabular}{l | cc | cc | cc|cc}
    \hline
    &\multicolumn{2}{c|}{CEGEN}&\multicolumn{2}{c|}{EDGAN}&\multicolumn{2}{c|}{TSGAN}&\multicolumn{2}{c}{COTGAN}\\
    Data &                  Disc & Pred & Disc & Pred & Disc & Pred & Disc & Pred\\
    \hline\hline
    Spot prices (d=2)    &\textbf{.014}&\textbf{.049}  &.137&\textbf{.049} &.066&.055  & .033&\textbf{.049}     \\
    Stocks (d=6)         &\textbf{.079}&\textbf{.040}  &.429&.041 &  .159&.041   & .116&.041\\
    Electric Load (d=12) & {.433}&{.028}               & .495&.046&   {.407}&.032  &  \textbf{.277}&\textbf{.022}   \\
    Jena climate (d=15)  & \textbf{.140}&\textbf{.032}  & .483&.035&  .179&\textbf{.032}  &  .227&.042 \\
    \hline
    \end{tabular}
    \caption{\textbf{Exp. D. } Discriminative and Predictive scores on real time series  (the lower, the better).}
    \label{tab:dispredscores}
\end{table}

\FloatBarrier
\section*{Conclusion}

We introduced three generative methods for times series, relying on a Deep Euler representation and Wasserstein distances. Two generative method EWGAN and EDGAN demonstrate an accuracy similar to state-of-the-art GAN generators and show better performance for capturing temporal dynamic metrics of the time series. The third method CEGEN is based on a loss metric computed on the conditional distributions of the time series. We prove that minimizing this loss ensures a proper estimation of the drift and volatility coefficients of underlying Itô processes.  Our experiments on synthetic and real-world datasets demonstrate that CEGEN outperforms the other generators marginal and temporal dynamics  metrics. CEGEN is able to capture correlation structures in high dimensions and is robust when combined with transfer learning on sparse datasets. Transfer learning tests show how this type of method can rely on a proven simulation model without replacing it completely. In further work, we plan to consider more specialized neural networks architectures for time series, extend our results to more general Lévy processes, which may include jumps, and consider not Gaussian noise.

\section*{Broader impact}
Generative methods for time series may be involved in industries using stochastic control and stochastic simulation methods making them of particular interest for the financial industry, for utilities and energy companies. When applied within a decision-making process, generative methods has to be used carefully as a failure during learning phase may lead to damageable consequences. In this situation, the outputs of the generators should not be left free, as this could lead to erratic optimal controls. Contrarily to the existing approaches which applies GANs and embedding to generate any kind of time series, we impose an Euler structure and we restrain ourselves within the (sufficiently) large class of Itô processes. One of the proposed algorithms manages to get good behavior for synthetic as well as for real data. Moreover, mathematical proofs gives an error estimate of the underlying process parameters for a given loss level. 
\bibliography{biblio}

\newpage 
\appendix
\section{Material of Section \ref{sec:deepconditionalmethod}}\label{Appx A}

\subsection{Counterexample}
\label{sec:counterexample}
Consider a timegrid $\{t_0,t_1,t_2\}$ with $\Delta t = t_{i+1} - t_i$. Let $b_X:(t,x)\rightarrow -2x/ \Delta t$, $b_Y:(t,x)\rightarrow 0$ and $\sigma_X(t,x) = \sigma_Y(t,x) = 1$ for all $x\in\mathbb{R}$. The $b$'s and $\sigma$'s are Lipschitz in the second coordinate:
\begin{eqnarray*}
\|b_X(t,x) - b_X(t, y) \|_2 \leq K \|x - y \|_2\\
\|b_Y(t, x) - b_Y(t, y) \|_2 \leq K \|x - y \|_2
\end{eqnarray*}

\noindent For the uni-dimentional case, we have:
\begin{table}[h]
\begin{tabular}{c c c}
$X_{t_0}=0$, & $X_{t_1} = \varepsilon^X_{t_1}$, & $X_{t_2} = -X_{t_1} + \varepsilon^X_{t_2}$\\
$Y_{t_0}=0$, & $Y_{t_1} = \varepsilon^Y_{t_1}$, & $Y_{t_2} = Y_{t_1} + \varepsilon^Y_{t_2}$
\end{tabular}
\end{table}

with $\varepsilon^X_{t_i},\varepsilon^Y_{t_i} \sim \mathcal{N}(0,\Delta t)$ and being i.i.d.
Then, for $i\in\{0,1,2\}$, we have $\mathcal{L}(X_{t_i})=\mathcal{L}(Y_{t_i})$ but $\mathbb{E}[X_{t_2}|X_{t_1}=z]=-z \neq z= \mathbb{E}[Y_{t_2}|Y_{t_1}=z]$

\subsection{Motivation - Details}
\label{prop:conditioning}
\begin{prop}\label{prop:conditioning_z}
Assume that for all $t_i\in \{t_0,\ldots,  t_N = T\}$, for all $z\in \mathbf{R}^d$, $X_{{t_i}+\Delta t}| X_{t_i}=z$ and $Y_{{t_i}+\Delta t}|  Y_{{t_i}}=z $ are identically distributed and that $\sigma_X(t_i,z) \sigma_X(t_i,z)^T$ (resp. $\sigma_Y(t_i,z) \sigma_Y(t_i,z)^T$) are positive semi-definite. Then $b_X(t_i,z) = b_Y(t_i,z)$ and $ \sigma_X(t_i,z)= \sigma_Y(t_i,z).$ 
\end{prop}
\begin{proof}
\label{sec:conditioning_z}
Let $t_i \in \{t_0,\dots , t_n = T \}$. For $z\in \mathbb{R}^d$.
We have, 
\begin{eqnarray*}
X_{{t_i} +\Delta t}| (X_{{t_i}} = z) &\sim& \mathcal{N}\left(z +  b_X({t_i},z) \Delta t,\sigma^2_X({t_i},z) \Delta t \right) \\
Y_{{t_i} +\Delta t}| (Y_{{t_i}} = z) &\sim& \mathcal{N}\left(z +  b_Y(t_i,z) \Delta t , \sigma^2_Y({t_i},z) \Delta t\right)
\end{eqnarray*}
then, $ b_Y(t_i,z) =  b_Y(t_i,z)$ and $\sigma_Y (t_i,z) \sigma_Y (t_i,z) ^T =  \sigma_X(t_i,z) \sigma_X(t_i,z) ^T$ for $z\in\mathbb{R}^d$. Matrix $\sigma_X({t_i},z) \sigma_X({t_i},z)^T$ being PSD, has a unique square root which is $\sigma_X (t_i,z)$. The same goes for $\sigma_Y (t_i,z)$. So, $\sigma_Y (t_i,z)= \sigma_X (t_i,z)$.
\end{proof}

\subsection{Proof of Proposition \ref{prop:inegalite}}\label{sec:proofconditionalwasserstein}
Assume that $\sigma^2_X(t_i, .)$, $\sigma^2_{Y^\theta}(t_i, .)$ are strictly positive and, together with $b_X(t_i, .)$ and $ b_{Y^\theta}(t_i, .)$, $K$-Lipschitz in their second coordinate. 
For $t_i\in\mathcal{T}$, let $(I_k)_k$ be a regular partition covering Supp$(X_{t_i})\cup$ Supp$(Y_{t_i})$ with mesh size $\Delta x$.
Let $\varepsilon>0$.\\
The Itô process $X$ follows the dynamics:
\begin{eqnarray*}
X_{t+\Delta t} =  X_{t} + b_X(t,X_t) \Delta t + \sigma_X (t,X_t) \mathcal{N}(0, \Delta t) \end{eqnarray*}
Thus, for all $z\in \mathbb{R}^d$, $X_{{t_i} +\Delta t}| (X_{{t_i}} = z) \sim \mathcal{N}\left(z +  b_X({t_i},z) \Delta t,\sigma^2_X({t_i} ,z)\Delta t \right)$ and the same goes for process $Y$.\\

Let $I_k= [a_k^1,a_{k+1}^1]\times [a_k^d,a_{k+1}^d]$. Suppose that for all $t_i\in \{t_0,\ldots,  t_N = T\}$, we have
\begin{eqnarray}
{\mathcal{W}}^2_2\left(\mathcal{L}(X_{t_{i+1}}| (X_{t_i}\in I_k)), \mathcal{L}(Y_{t_{i+1}}| ( Y_{{t_i}}\in I_k)\right)\leq \varepsilon
\end{eqnarray}

then by definition of ${\mathcal{W}}^2_2$, $\exists z_1, z_2\in I_k$, such that: 
\begin{eqnarray}
    \| z_1 + b_X(t_i, z_1)\Delta t - z_2 - b_Y(t_i, z_2) \Delta t\|_2^2 \nonumber\\
    +\mathcal{B}^2(\sigma^2_X(t_i, z_1)\Delta t , \sigma^2_Y(t_i, z_2)\Delta t )\leq \varepsilon \label{eq:inegalite1}
\end{eqnarray}

By standard norm inequalities and Lipschitz properties of $b_X(t_i,.)$ and $b_Y(t_i,.)$, we bound the squared distance of drifts for all $z\in I_k$ and with mesh grid $\Delta x = \|a_{k+1} - a_k\|_2$.
$$\| \left(b_X({t_i},z_1) -  b_Y(t_i,z_2)\right)\Delta t + (z_1 - z_2) \|_2^2 \leq \varepsilon$$

\begin{eqnarray*}
\|b_X({t_i},z_1)-  b_Y(t_i,z_2)\|_2 &\leq& \frac{\sqrt{\varepsilon} + \Delta x}{\Delta t}  \\ 
\|b_Y(t_i,z_2) -b_X({t_i},z) \|_2  &-&\|b_X({t_i},z_1) - b_X(t_i,z)\|_2 \nonumber\\
&\leq& \frac{\sqrt{\varepsilon}+ \Delta x}{\Delta t}\\ 
\|b_Y(t_i,z_2) -b_X({t_i},z) \|_2 &\leq& \frac{\sqrt{\varepsilon}+ \Delta x}{\Delta t} + K\Delta x\\ 
\|b_Y(t_i,z) - b_X({t_i},z)\|_2   &\leq& \frac{\sqrt{\varepsilon}  +\Delta x}{\Delta t} + 2K \Delta x
\end{eqnarray*}

We recall that the Bures metrics \citep{bhatia2019bures,malago2018wasserstein} between positive definite matrices $A$ and $B$ is defined by
$$\mathcal{B}^2(A, B) {=} Tr(A) + Tr(B) - 2 Tr(A^\frac{1}{2}BA^\frac{1}{2})^{1/2}$$
For volatility bound, from Equation (\ref{eq:inegalite1}) we have,
In the \textbf{d=1} case, this implies $\|\sigma_X(t_i, z_1)\sqrt{\Delta t} -  \sigma_Y(t_i, z_2)\sqrt{\Delta t}\|_2^2 \leq \varepsilon$ which leads to : for all $z\in I_k$,
\begin{eqnarray}
\|\sigma_X(t_i, z) - \sigma_Y(t_i, z)\|_2 \leq \sqrt{\frac{\varepsilon}{\Delta t}} + 2K\Delta x  \nonumber
\end{eqnarray} 
For \textbf{d$>$1}, let's denote $\mathcal{H}$ the Hellinger distance between positive density matrices: 
\begin{eqnarray}
\mathcal{H}(A, B) {=} \|A^\frac{1}{2} - B ^\frac{1}{2}\|_2 
\end{eqnarray}
For two density matrices $A$ and $B$, from \citep{spehner2017geometric} (Equation 74) 
we have $\mathcal{H}(A, B) < \sqrt{2}\mathcal{B}(A, B)$. Following the trace assumption, $Tr(\sigma^2_X(t_i, z_1)) = Tr(\sigma^2_Y(t_i, z_2)) = \alpha$ and then, 
\begin{eqnarray}
\mathcal{H}\left(\frac{\sigma^2_X(t_i, z_1)}{\alpha}\Delta t, \frac{\sigma^2_Y(t_i, z_2)}{\alpha}\Delta t\right) &\leq& \sqrt{2\varepsilon} 
\end{eqnarray}
Thus we get,
\begin{eqnarray}
\|\sigma_X(t_i, z_1) - \sigma_Y(t_i, z_2)\|_2  \leq \sqrt{\frac{2\alpha \varepsilon}{\Delta t }} .
\end{eqnarray}
In particular, using the $K$-Lipschitz property of the volatility functions, we obtain: for all $z\in I_k$,
\begin{eqnarray}
\|\sigma_X(t_i, z) - \sigma_Y(t_i, z)\|_2 \leq \sqrt{\frac{2\alpha\varepsilon}{\Delta t }}   + 2K\Delta x
\end{eqnarray}

\begin{remark} The Proposition above extends to the Wasserstein-2 loss if we conditionate from points instead of conditioning from intervals. Indeed, suppose for all $t_i\in \{t_0,\ldots,  t_N = T\}$, we have
\begin{eqnarray}
\mathcal{W}^2_2\left(\mathcal{L}(X_{t_{i+1}}| (X_{t_i}=z_1), \mathcal{L}(Y_{t_{i+1}}| ( Y_{{t_i}}=z_2)\right)\leq \varepsilon
\end{eqnarray}
then, as we have : 
\begin{eqnarray*}
X_{{t_i} +\Delta t}| (X_{{t_i}} = z_1) &\sim& \mathcal{N}\left(z_1 +  b_X({t_i},z_1) \Delta t,\sigma^2_X({t_i},z_1)\Delta t \right) \\
Y_{{t_i} +\Delta t}| (Y_{{t_i}} = z_2) &\sim& \mathcal{N}\left(z_2 +  b_Y(t_i,z_2) \Delta t , \sigma^2_Y({t_i},z_2)\Delta t\right)
\end{eqnarray*}
We can use the closed form of the Gaussian expression of Wasserstein-2:
\begin{eqnarray}
    \| z_1 + b_X(t_i, z_1)\Delta t - z_2 - b_Y(t_i, z_2) \Delta t\|_2^2\nonumber\\
    + \mathcal{B}^2(\sigma^2_X(t_i, z_1)\Delta t , \sigma^2_Y(t_i, z_2)\Delta t )\leq \varepsilon 
\end{eqnarray}
Similar results as proof \ref{prop:inegalite} follow. 
\end{remark}

\section{Additional numerical results}\label{sec:numericalAppendix}

\begin{figure}
    \centering
    \includegraphics[width=0.7\linewidth]{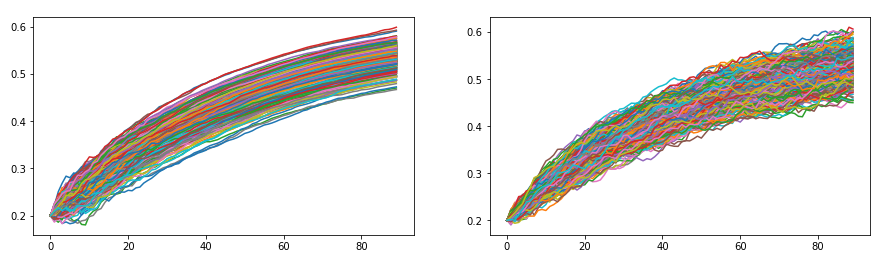}
    \caption{\textbf{Left:} Unsatisfactory generations from GAN \textbf{Right:} Reference (Ornstein-Uhlenbeck)}
    \label{fig:sinkhorn}
\end{figure}

\subsection{Experiment A - Unidimensional case with synthetic data (Exp. A) }\label{sec:additionalExpA}
Quantitative evaluations provided by Table \ref{tab:envelope} for each models highlight how CEGEN and TSGAN stand out in term of average moments accuracy for a 1-dimensional Black-Scholes. We can see that the benchmark TSGAN and our model CEGEN are faithful to synthetic trajectories and outperform both Euler GANs. However, CEGEN is also able to capture the temporal dynamics of both processes, as QVar metrics reports.


\begin{table}[h!]
    \centering
    \resizebox{.5\textwidth}{!}{
    \begin{tabular}{c cccc}
    \hline
    &\multicolumn{4}{c}{\textbf{Black-Scholes}}\\ 
    metrics &CEGEN&EWGAN&EDGAN&TSGAN\\\hline\hline
    q05&9.95e-05&2.55e-02&3.49e-03&\textbf{2.61e-06}\\
    Avg&\textbf{8.01e-07}&4.04e-02&2.20e-05&9.44e-07 \\
    q95&\textbf{7.85e-06}&6.02e-02&2.85e-03&2.47e-05  \\
    QVar&\textbf{4.54e-04}  & 7.30e-02  &  4.12e-02  & 2.38e+00\\\hline
    &\multicolumn{4}{c}{\textbf{Ornstein-Uhlenbeck}}\\
    metrics &CEGEN&EWGAN&EDGAN&TSGAN\\\hline\hline
    q05&  4.89e-04&3.98e-03&7.30e-02&\textbf{2.96e-06}\\
    Avg&  \textbf{2.27e-07}&2.39e-05&4.47e-02&2.25e-06\\
    q95&  8.55e-04&3.98e-03&2.38e-02&\textbf{6.49e-06}\\
    QVar& 4.59e-03  & 2.51e+00  & \textbf{1.67e-03}  & 1.04e+00\\\hline
    \end{tabular}
    }
    \caption{\textbf{Exp. A} Mean squared error (MSE) between reference samples and generated time series on marginal metrics.}
    \label{tab:envelope}
\end{table}

\subsection{Experiment B - Multidimensional case with synthetic data (Exp. B)}\label{sec:additionalExpB}
Figure \ref{fig:egenbures20samp} reports envelope of samples from CEGEN model (orange) on a 20-dimensional Black-Scholes (blue). Full lines give the marginal averages over time, and dash lines give average 5\% and 95\% quantiles respectively. Our generator is still able to retrieve faithfully the average moments in high dimension. This is confirmed with quantitative evaluation provided by Table \ref{tab:envelopedimensions} where the benchmark TSGAN and CEGEN stand out compared to Euler GANs.\\
Empirical correlation matrices are also retrieved by CEGEN up to dimension 20. In Figure \ref{fig:correlation20}, we represent the reference empirical correlation alongside the one coming from generated samples in both correlated and independent case. The term by term mean squared error of correlation matrices (the more black, the better) confirms that CEGEN generations are highly realistic.

\begin{figure*}[!ht]
\centering 
  \includegraphics[width=0.7\linewidth]{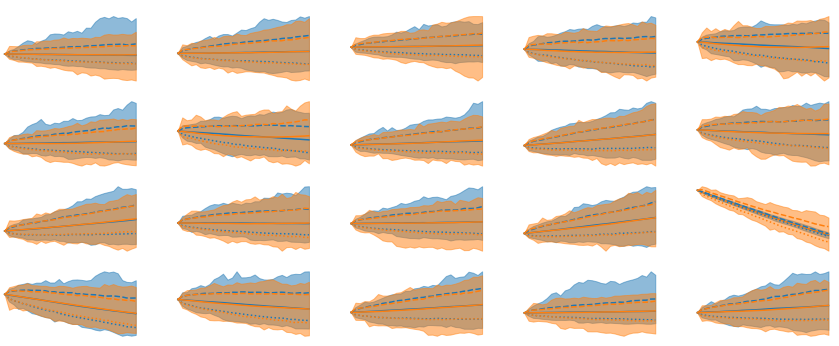}
  \caption{\textbf{Exp. B.} Samples from CEGEN model on 20-dimensional Black-Scholes model-based generation. Envelope of both CEGEN (orange) and Reference (blue) time series. }
   \label{fig:egenbures20samp}
\end{figure*}

\begin{table}[h!]
    \centering
    \footnotesize
    \begin{tabular}{ccccc}
    \hline
    &CEGEN&EWGAN&EDGAN&TSGAN\\\hline\hline
    \multicolumn{5}{c}{Dimension = 4 }\\ \hline 
    q05&7.58e-05&\textbf{2.24e-05}&1.39e-03&3.25e-05\\
    Mean&6.02e-06&1.34e-05&1.12e-05&\textbf{4.21e-06}\\
    q95&4.46e-05&1.38e-04&1.23e-03&\textbf{3.31e-05}\\
    \multicolumn{5}{c}{Dimension = 10 } \\ \hline 
    q05&\textbf{1.63e-04}&5.23e-04&1.86e-03&2.00e-03\\
    Mean&\textbf{6.59e-06}&2.60e-05&1.27e-05&5.12e-04\\
    q95&\textbf{2.80e-04}&9.17e-04&1.91e-03&6.25e-03\\
    \multicolumn{5}{c}{Dimension = 20 } \\ \hline 
    q05&\textbf{1.16e-04}&9.51e-04&3.70e-03&1.19e-04\\
    Mean&\textbf{1.19e-05}&4.88e-05&3.71e-05&1.62e-05\\
    q95&4.35e-04&1.88e-03&2.38e-03&\textbf{4.14e-04}\\
    \end{tabular}
    \caption{\textbf{Exp. B. }Mean squared error (MSE) between reference and generated envelope statistics on a BS case}
    \label{tab:envelopedimensions}
\end{table}

\begin{figure*}[h!]
    \centering
    \includegraphics[width=0.49\linewidth]{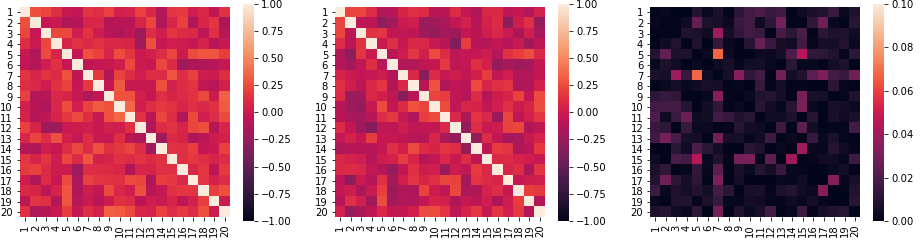}
    \includegraphics[width=0.49\linewidth]{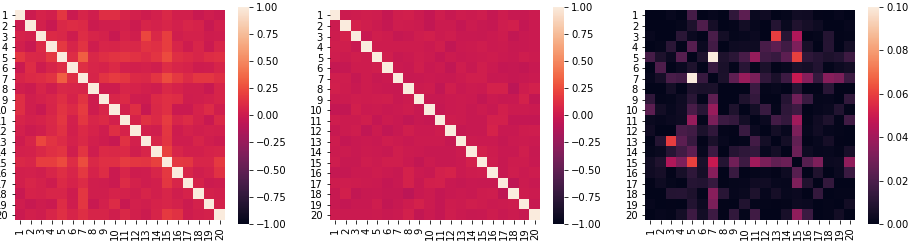}
    \caption{\textbf{Exp. B} Mean over time of empirical correlation matrices illustrates a diagonal covariance matrix case (independence case). First heatmap is generated samples from CEGEN (left), second is Monte Carlo ones (right), the target, the third heatmap (mostly black) represents the mean squared error of the two correlation matrices. }
    \label{fig:correlation20}
\end{figure*}


\subsection{Experiment C - Transfer Learning (Exp. C)}\label{sec:additionalExpC}
We provide in Table \ref{tab:adpatparam} the empirical coefficients of the Ornstein-Ulhenbeck we try to generate. The CEGEN algorithm benefeting of transfer learning gives the closer estimation to the real parameters. However, the CEGEN only trained on the few available samples proposes a better estimation of $\mu$ term. A deeper analysis would be welcome and is the subject of future work.

\begin{table}[h!]
    \centering
    \begin{tabular}{lccc}
    \hline      &Misspecified & CEGEN & CEGEN \\
                &   model    & w/o transfer & with transfer\\\hline\hline
    ${\theta}$ (2.00)   &3.00 & 1.54 &\textbf{1.78}\\
    ${\mu}$ (0.60)      &0.80 & \textbf{0.65} &1.01\\
    ${\sigma}$ (0.15)  &0.10 & -0.03   &\textbf{0.15}\\
    \hline
    \end{tabular}
    \caption{\textbf{Exp. C.} Empirical Ornstein-Uhlenbeck coefficient estimations according to misspecifed samples, CEGEN trained on few original data and CEGEN with transfer learning.}

    \label{tab:adpatparam}
\end{table}


\subsection{Experiment D - Real data and other benchmarks (Exp. D)}\label{sec:additionalExpD}
Table \ref{tab:morebenchmarks} reports discriminative and predictive performances of EWGAN, the conditional recurrent RCGAN \citep{esteban2017real} and an unconditional MMD with Gaussian kernel GMMN \citep{li2015generative}.
\begin{table}[h!]
    \centering
    \footnotesize
    \resizebox{.48\textwidth}{!}{
    \begin{tabular}{l | cc|cc|cc}
    \hline
    &\multicolumn{2}{c|}{EWGAN}&\multicolumn{2}{c|}{RCGAN}&\multicolumn{2}{c}{GMMN}\\
    Data &                  Disc & Pred&Disc & Pred&Disc & Pred\\
    \hline\hline
    Spot prices (d=2)    &.225&.050  &  .427&.809  &  .137&.671\\
    Stocks (d=6)         &.238&.042  &  .287&.616  &  .499&.626\\
    Electric Load (d=12) &.410&.029  &  .495&.581  &  .499&.566\\
    Jena climate (d=15)  &.479&.034  &  .499&.651  &  .295&.634\\
    \hline
    \end{tabular}
    }
    \caption{\textbf{Exp. D. } Discriminative and Predictive scores on real time series  (the lower, the better).}
    \label{tab:morebenchmarks}
\end{table}

\section{Algorithms details}\label{sec:algorithm}
In this section, we detail the pseudo algorithms of both Euler GANS (EWGAN and EDGAN). We also explain more deeply the conditional loss computation, as well as some tested variants.

\begin{algorithm}
\caption{Euler Wasserstein-1 Generative Adversarial Networks (EWGAN).}
\label{alg:ewgan}
\begin{algorithmic}
\STATE {\bfseries Input:} $\theta_0$, $\varphi_0$ randomly chosen, $\alpha$, $\beta$, learning rates,
\STATE $K$ number of iterations,$M$ batch size,
$n_{c}$ critic iterations, $c$ clipping value, $\left(X^{(i)}\right)_{i=1..M}$ real data
\STATE {\bfseries Output:}{$\theta, \varphi$}\\
$\theta \leftarrow \theta_0, \varphi \leftarrow \varphi_0$\;
\WHILE{NOT CONVERGE}
    \FOR{$j=1..n_{\text{critic}}$}
        \STATE $x \leftarrow M$ samples with $X^{(i)} = (X^{(i)}_{t_1},\ldots , X^{(i)}_{t_N})_{i=1..M} $\;
        \STATE $z \leftarrow M$ samples iid gaussian noise
        \STATE $y^\theta \leftarrow M$ generation from Euler scheme and $g_\theta(z)$\;
        \STATE $\varphi \leftarrow \varphi + \alpha \text{Adam}\left( \nabla_\varphi( \mathbb{E}[d_\varphi(x)] - \mathbb{E}[d_\varphi(y^\theta)]), \alpha \right) $
        \STATE $\varphi \leftarrow$ gradient penalty $(\varphi, 10)$\;
    \ENDFOR
    \STATE $x \leftarrow M$ samples with $X^{(i)} = (X^{(i)}_{t_1},\ldots , X^{(i)}_{t_N})_{i=1..M} $\;
    \STATE $z \leftarrow M$ samples iid Gaussian noise\;
    \STATE $Y^\theta \leftarrow M$ generation from Euler scheme and $g_\theta(z)$\;
    \STATE $\theta \leftarrow \theta - \beta \text{Adam}\left( \nabla_\theta \mathbb{E}[d_\varphi(Y^\theta)], \beta\right)$\;
    \ENDWHILE
\end{algorithmic}
\end{algorithm}

\begin{algorithm}
\caption{Euler Dual Generative Adversarial Networks (EDGAN).}
\label{alg:edgan}
\begin{algorithmic}
\STATE {\bfseries Input:} $\theta_0$, $\varphi_0$ randomly chosen, $\alpha, \beta, \gamma$ learning rates,
\STATE $K$ number of iterations, $M$ batch size,
$n_{c}$ critic iterations, $c$ clipping value, $\left(X^{(i)}\right)_{i=1..M}$ real data
\STATE {\bfseries Output:}{$\theta, \varphi, \psi$}\;
$\theta \leftarrow \theta_0, \varphi \leftarrow \varphi_0, \psi \leftarrow \psi_0$,;
\WHILE{NOT CONVERGE}
    \FOR{$j=1..n_{\text{critic}}$}
        \STATE $x \leftarrow M$ samples of $X^{(i)} = (X^{(i)}_{t_1},\ldots , X^{(i)}_{t_N})_{i=1..M} $\;
        \STATE $z \leftarrow M$ samples iid Gaussian noise\;\\
        $y^\theta \leftarrow M$ generations from Euler scheme and $g_\theta(z)$\;
        \;\\
        \STATE $\varphi \leftarrow \varphi + \alpha \text{Adam}\left( \nabla_\varphi( \mathbb{E}[d_\varphi(x)] - \mathbb{E}[d_\varphi(y^\theta)]); \alpha \right) $
        \STATE $\varphi \leftarrow$ gradient penalty $(\varphi, 10)$\;
        \;\\
        \FOR{$t=t_1..t_N$}
            \STATE $\psi \leftarrow \psi + \gamma \text{Adam}\left( \nabla_\psi( \mathbb{E}[d_\psi(x_t)] - \mathbb{E}[d_\psi(y^\theta)_t]); \gamma \right) $
            \STATE $\psi \leftarrow$ gradient penalty $(\psi, 10)$\;
        \ENDFOR
    \ENDFOR
    \STATE $x \leftarrow M$ samples with $X^{(i)} = (X^{(i)}_{t_1},\ldots , X^{(i)}_{t_N})_{i=1..M} $\;
    \STATE $z \leftarrow M$ samples iid Gaussian noise\;
    \STATE $Y^\theta \leftarrow M$ generations from Euler scheme and $g_\theta(z)$\;
    \STATE $\theta \leftarrow \theta - \beta \text{Adam}\left( \nabla_\theta \mathbb{E}[d_\varphi(Y^\theta)], \beta\right)$\;
    \ENDWHILE
\end{algorithmic}
\end{algorithm}

For given loss $\ell$ (Bures-Wasserstein (\ref{eq:newwasserstein}) in the paper), we compute the conditional loss by extracting the elements at a certain date such that the previous state belongs to an ensemble $I$. We propose below two ways to do it.\\
The first approach consists in sorting each dimension at each time step in order to get $K$ quantiles $(a_k)_{k=1..K}$ of both $X_{t_i}$ and $Y_{t_i}$. At date $t_i$ for $i \in\{1,...,T\}$ and for each dimension $d\in\{1,...,D\}$, for a given batch of samples, $\mathcal{L}(X_{t_{i+1}}\,|\,X_{t_{i}}\in I)$ is approximated by selecting only the realizations $x^d_{t_{i+1}}$ such that the previous state $x^d_{t_{i}}$ belongs to the interval $I^d_k = [a^d_k, a^d_{k+1}]$. The losses $\ell^d_k$ between the two conditional distributions are then summed up over all dimensions and subdivisions. To take into account the disjoint support case (for instance samples $x^d_{t_{i}} \in [-1,0[$ and $y^d_{t_{i}} \in ]0,1]$), we penalize by the distance separating the supports. See Algorithm \ref{alg:cdqloss} for further details.\\
Another approach is to compute the partitions of Supp$(X^d_{t})$ before the generator training phase. We use $T$ Kmeans to compute the centers of $K$ clusters at each time step. Then, during the generator training we compute the loss $\ell$ between $X_{t_{i+1}}$ and $Y_{t_{i+1}}$ such that their respective previous states belong to the same cluster $k$. This method has the advantage that the generated samples share the same support as the real data one.
We use the conditional loss by disjoint quantiles in our experiments, because the algorithm runs faster and gives better empirical results.

\begin{algorithm}[h]
  \caption{Conditional Loss by disjoint quantiles.}
  \label{alg:cdqloss}
  \footnotesize
  \begin{algorithmic}
    \STATE {\bfseries Input:} processes of length $T$ $X = (X^1, \ldots, X^D)$, $Y = (Y^1,\ldots,Y^D)$, $\lambda$\\
    \FOR{$t=1...T$}
        \FOR{$d=0...D$}
            \STATE $I^d_{K,x}\leftarrow  K$ subdivisions of Supp$(X^d_{t})$;
            \STATE $I^d_{K,y}\leftarrow  K$ subdivisions of Supp$(Y^d_{t})$;
            \FOR{$k=0...K$}
                \IF{Supp$(X^d_{t})\cup$  Supp$(Y^d_{t}) \neq \emptyset$}
                    \STATE $\ell^d_{t+1, k}\leftarrow {\mathcal{W}}^2_2(\mathcal{L}(X_{t+1}| X^{d}_{t} \in I^d_{K,x}), \mathcal{L}(Y_{t+1}| Y^{d}_{t} \in I^d_{K,y}))$
                \ELSE
                    \STATE $\ell^d_{t+1, k}\leftarrow \lambda |\mathbb{E}[X^{d}_{t}] - \mathbb{E}[Y^{d}_{t}] |$
                \ENDIF
            \ENDFOR
        \ENDFOR
      \ENDFOR
      \STATE $\ell =\sum_{t=1}^{T} \sum_{d=1}^{D} \sum_{k=1}^{K} \ell^d_{t+1,k}$ 
    \STATE {\bfseries Output: $\ell$}
\end{algorithmic}
\end{algorithm}

\FloatBarrier

\section{Models and hyperparameters}\label{sec:hyperparameters}
We use tensorflow to implement neural networks. The networks architecture is composed of 3-layers of 4 times the data dimension neurons each (for stocks 4$\times$6=24 neurons). Euler generator networks are feed-forward, as we want to be Markovian, while benchmarks TSGAN and COTGAN architecture benefit of recurrent networks (GRU, LSTM). Code of TSGAN is available online \href{https://github.com/jsyoon0823/TimeGAN}{(\textbf{link})}, as well as the code of COTGAN \href{https://github.com/tianlinxu312/cot-gan}{(\textbf{link})}.
Other details are precised in Table \ref{tab:hyperparameters}.
Real dataset are normalized with MinMax scaler ((x- min)/(max -min)) and the first date always starts at spot $X_0=0.2$.


\begin{table}[h]
    \centering
    \begin{tabular}{cc}
    \hline
    \multicolumn{2}{c}{Settings of neural networks}\\\hline\hline
    T (ndates) & 30\\
    White noise dim & (T$\times$d)\\
    Optimizer & Adam\\
    Nb iterations & 5000\\
    Batch size & 300\\
    Learning rates & $1.10^{-3}$\\
    \hline
    \end{tabular}
    \caption{Neural network hyper-parameters}
    \label{tab:hyperparameters}
\end{table}

To compute the discriminative and predictive scores, we use the same network architecture and parameters as \citep{yoon2019time} (actually we use their code). The neural networks are 2-layer LSTMs with hidden dimensions 4 times the size of the input features, and use tanh as the activation function and sigmoid as the output layer activation function (such that output belongs to [0,1]).

The training is done on 12 i7-9750H processors at 2.60 GHz.

\section{Results variation}\label{sec:resultvariation}
Table \ref{tab:errorbar} illustrates the variation between three different trainings of each generators for stocks data. We recall that the discriminative and predictive score are obtained by training 10 LSTM networks and averaging their scores (they thus include some additional variation).
\begin{table}[!h]
    \centering
    \begin{tabular}{c cc}
    \hline
    Stocks data & Discriminative & Predictive\\\hline
    EWGAN&  .417($\pm$.041) & .041($\pm$.001)\\
    EDGAN&  .444($\pm$.146) & .041($\pm$.000)\\
    CEGEN&  .077($\pm$.015) & .040($\pm$.000)\\
    TSGAN&  .168($\pm$.025) & .041($\pm$.001)\\
    COTGAN& .094($\pm$.022) & .041($\pm$.000)\\
    \hline
    \end{tabular}
    \caption{Performance variations of each generator for stocks data on discriminative and predictive scores, for three different trainings.}
    \label{tab:errorbar}
\end{table}

\section{Data}\label{sec:data}
\begin{table}[!h]
    \centering
    \begin{tabular}{c ccc}
    \hline
    Dataset       & Sequences & Seq. length & Dim.\\\hline\hline
    Price         & 52608  & 30 & 2\\
    Stocks        & 3600   & 30 & 6\\
    Electric Load* & 50000 & 30 & 13\\
    Jena Climate*       & 50000 & 30 & 15\\
    \hline
    \end{tabular}
    \caption{Data description.}
    \label{tab:data}
\end{table}
Table \ref{tab:data} reports the number of observations of each dataset, the sequence length chosen in our experiments, as well as their dimension.
All datasets are available online, and can be downloaded from: RTE for electric load and price \href{https://www.services-rte.com/fr/telechargez-les-donnees-publiees-par-rte.html?category=generation&type=actual_generations_per_production_type}{(\textbf{link})}, Keras for Jena climate \href{https://storage.googleapis.com/tensorflow/tf-keras-datasets/jena_climate_2009_2016.csv.zip}{(\textbf{link})}. Stocks data source are described in \citep{yoon2019time}.
*For Electric Load and Jena Climate we take only the first 50000 observations.

\end{document}